\documentclass[11pt]{article}

\usepackage{amsthm,amsmath, amssymb, latexsym}
\usepackage{fullpage}
\usepackage[numbers]{natbib}
\usepackage{algorithm, algorithmic}
\usepackage{enumitem}
\usepackage{pdfsync}
\usepackage{graphicx}
\usepackage{xcolor}
\usepackage{fancyhdr}
\usepackage{yfonts,mathrsfs,skak}
\usepackage[mathcal]{euscript}
\usepackage{rotating}
\usepackage[colorlinks=true,linkcolor=blue,citecolor=violet]{hyperref}

\let\inf\undef
\DeclareMathOperator*{\inf}{\vphantom{p}inf}
\let\sup\undef
\DeclareMathOperator*{\sup}{\vphantom{p}sup}
\let\max\undef
\DeclareMathOperator*{\max}{\vphantom{p}max}
\let\min\undef
\DeclareMathOperator*{\min}{\vphantom{p}min}

\usepackage{MnSymbol}

\DeclareMathOperator*{\Ex}{\vphantom{p}\mathbb{E}}
\let\inf\undef
\DeclareMathOperator*{\inf}{\vphantom{p}inf}
\let\sup\undef
\DeclareMathOperator*{\sup}{\vphantom{p}sup}

\newcommand{\mbf}[1]{\mathbf{#1}}

\newcommand{\mrm}[1]{\mathrm{#1}}

\newcommand{\gap}{\mrm{gap}}
\newcommand{\pred}{\widehat{{y}}}

\newcommand{\argmin}[1]{\underset{#1}{\mrm{argmin}} \ }

\newcommand{\reals}{\mathbb{R}}
\newcommand{\En}{\mathbb{E}}  

\newcommand{\conv}{\operatorname{conv}}

\newcommand{\ind}[1]{{\bf{I}}\left\{#1\right\}}
\newcommand{\tr}{\ensuremath{{\scriptscriptstyle\mathsf{T}}}}

\newcommand{\brho}{{\boldsymbol{\rho}}}

\newcommand{\bepsilon}{\boldsymbol{\epsilon}}
\newcommand{\e}{\boldsymbol{e}}

\newcommand\cD{\mathcal{D}}

\newcommand\cM{\mathcal{M}}

\newcommand\X{\mathcal{X}}

\newcommand\F{\mathcal{F}}

\newcommand\Rad{\mathfrak{R}}

\newcommand\Reg{\mbf{Reg}}

\newcommand{\Rel}[2]{\mbf{Rel}_{#1}\left(#2 \right)}
\newcommand{\Relfull}[2]{\mbf{Rel}^{\dagger}_{#1}\left(#2 \right)}

\newcommand{\multiminimax}[1]{\ensuremath{\left\llangle #1\right\rrangle}}

\def\deq{\triangleq}

\newcommand{\Mhat}{{\widehat{\cM}}}

\newtheorem{theorem}{Theorem}
\newtheorem{lemma}[theorem]{Lemma}

\theoremstyle{definition}
\newtheorem{definition}{Definition}

\usepackage{bbm}

\title{BISTRO: An Efficient Relaxation-Based Method for Contextual Bandits}
\author{Alexander Rakhlin \\ University of Pennsylvania \and Karthik Sridharan \\ Cornell University}

\date{\today}
\makeindex

\begin{document}
\maketitle

\begin{abstract}
	We present efficient algorithms for the problem of contextual bandits with i.i.d. covariates, an arbitrary sequence of rewards, and an arbitrary class of policies. Our algorithm BISTRO requires $d$ calls to the empirical risk minimization (ERM) oracle per round, where $d$ is the number of actions. The method uses unlabeled data to make the problem computationally simple. When the ERM problem itself is computationally hard, we extend the approach by employing multiplicative approximation algorithms for the ERM. The integrality gap of the relaxation only enters in the regret bound rather than the benchmark. Finally, we show that the adversarial version of the contextual bandit problem is learnable (and efficient) whenever the full-information supervised online learning problem has a non-trivial regret guarantee (and efficient).
\end{abstract}

\section{Introduction}

A multi-armed bandit with covariates (also known as a \emph{contextual bandit}) is a generalization of the classical multi-armed bandit problem \cite{lai85asymptotically}. As the name suggests, in this natural formulation the quality of the arms may depend on the observed set of covariates. Contextual bandits arise in many application areas, from ad placement and news recommendation to personalized medical care and clinical trials. In recent years, there has been a strong push to develop computationally efficient regret minimization methods with respect to a given set of policies \cite{langford2008epoch,dudik2011efficient,BeyLanLiReySch11,agarwal2014taming}. The grand goal here would be to develop efficient and statistically optimal methods for large (and possibly uncountable) sets of policies, just as machine learning and statistics succeeded in developing methods that perform well relative to rich classes of predictors (linear separators, SVMs, and so forth). Compared to batch learning, however, the state of affairs at the moment is quite poor. It appears to be difficult to develop scalable methods even for a finite set of policies, as witnessed by the papers mentioned earlier. To some extent, the reason is not surprising: while in statistical learning the batch nature of the problem suggests the empirical objective to optimize, the scope of algorithms for contextual bandits is not at all clear. 

\citep{agarwal2014taming} exhibit a computationally attractive method for a finite class of policies, given an ERM (empirical risk minimization) oracle for the class. The oracle model allows one to address the question of how much more difficult (computationally) the bandit problem is in comparison to the batch learning problem. 

In the present paper, we introduce a family of efficient methods (and, more generally, a new algorithmic approach based on relaxations) for minimizing regret  against a potentially uncountable class $\F$, given that the \emph{value} of the ERM objective can be computed. In addition, we require access to i.i.d. draws of contexts (e.g. unlabeled data) --- a realistic assumption in many application areas mentioned earlier. Our method requires only $d$ oracle calls per round, irrespective of the size of the policy class. Furthermore, the results hold in the hybrid scenario where the contexts are i.i.d. but rewards evolve according to an arbitrary process.

Let us now describe the scenario in more detail. On each round $t=1,\ldots,n$, we observe covariates $x_t\in\X$, select an action $\pred_t\in\{1,\ldots,d\}\deq [d]$, and observe the cost $c_t(\pred_t)$ of the chosen action. Here $c_t\in[0,1]^d$ is a cost assignment to all actions, chosen by Nature independently of $\pred_t$. This cost vector remains unknown to us, except for the coordinate $c_t(\pred_t)$. Since we include randomized prediction methods, we denote the distribution over the $d$ choices on round $t$ by $q_t\in\Delta_d$, and draw $\pred_t\sim q_t$. The goal is to design a prediction method with small expected cumulative cost ~$\sum_{t=1}^n q_t^\tr c_t$.

We assume that $x_1,\ldots,x_n$ are drawn i.i.d. from some unknown distribution $P_x$ on $\X$. At the same time, we do not place any assumption on the sequence of costs $c_1,\ldots,c_n$, which may evolve according to some arbitrary stochastic process, or be an ``individual sequence,'' or even be chosen adaptively and adversarially. As such, our setting may be termed ``hybrid i.i.d.-adversarial.'' Our results also hold in the so-called transductive setting, where the side information is presented ahead of time.\footnote{In Section~\ref{sec:extensions} we also discuss the fully-adversarial case (see \cite{auer2002nonstochastic,mcmahan2009tighter} for the famous EXP4 algorithm for finite $\F$).}

We have in mind machine learning applications such as online ad or product placement, whereby the contextual information $x_1,\ldots,x_n$ of website visitors may be viewed as an i.i.d. sequence, \emph{yet the decisions made by these customers might be too complex to be described in a probabilistic form.}

A common way to encode the prior knowledge about the problem is to take a class $\F$ of functions (or, \emph{deterministic policies}) $\X\to[d]$, with the hope that one of the functions will incur small cost on the presented contexts. With this ``inductive bias,'' we then aim to make predictions as to minimize regret
\begin{align}
	\label{eq:regret}
	\Reg = \sum_{t=1}^n q_t^\tr c_t-\inf_{f\in\F} \sum_{t=1}^n f(x_t)^\tr c_t,
\end{align}
where henceforth we abuse the notation by identifying the value $f(x)\in[d]$ with the standard basis vector $\e_{f(x)}$.
This regret formulation encodes the prior knowledge of the practitioner. If the modeling choice $\F$ is good and \eqref{eq:regret} is small, the algorithm is guaranteed to incur small loss $\sum_{t=1}^n q_t^\tr c_t$. Modeling the set of solutions $\F$ to the problem is a more direct approach (in the spirit of statistical learning) as compared to the harder problem of positing distributional assumptions on the relationship between contexts and the rewards. (The latter approach typically suffers from the curse of dimensionality.)

The difficulty of the problem arises from the form of the feedback. The customer seeking to buy a product different from what is presented by the recommendation engine may leave the site without revealing her valuation for all the items. Similarly, in personalized care, we may only observe the effect of the drug choice selected for the given patient. It is well recognized that  exploration---or randomization---is required in these problems. Yet, in the contextual bandit setting the exploration-exploitation trade-off is not simple, as the quality of the arms changes with the context in a way that is only indirectly captured by the benchmark term.

Online multiclass classification with one bit (correct-or-not) feedback can be seen as an example of our setting. In that case $c_t$ is a standard basis vector $e_{y_t}$ for some class $y_t\in[d]$, and the feedback is $c_t(\pred_t)=\ind{\pred_t\neq y_t}$. Unlike \cite{kakade2008efficient}, we posit that side information is i.i.d.---an assumption that will play a key role in developing computationally efficient methods, even for the indicator (rather than the easier hinge) loss.

The hybrid i.i.d.-adversarial scenario has been studied in both the full information and contextual bandit settings in \cite{lazaric2009hybrid}. Their algorithm, as well as the algorithm of \cite{BeyLanLiReySch11}, maintain distributions over the set of functions and, hence, computation can be linear in the size of $\F$.

For the case when $\F$ is finite, the upper bound for BISTRO provided in Theorem~\ref{thm:main} is $O(n^{3/4}(\log|\F|)^{1/4})$. The work of \cite{agarwal2014taming} gives a better $O(n^{1/2}(\log|\F|)^{1/2})$ rate for the case when rewards are i.i.d. On the other hand, our results hold for
\begin{itemize}
	\item arbitrary $\F$ and arbitrary reward sequences, 
	\item approximate ERM values and a way to address the computational problem associated to ERM.
\end{itemize}

We remark that if contexts are arbitrary as well, our setting subsumes the problem of multiclass prediction with bandit feedback and indicator loss, as described above. Even for the multiclass hinge loss, it is still unclear (at least to the authors) whether the rate $O(n^{2/3})$ for the linear classifier considered in \cite{kakade2008efficient} can be improved.\footnote{The $O(n^{1/2})$ rate in \cite{hazan2011newtron} is only proved for the case of log-loss.} It is, therefore, an open question whether the $O(n^{3/4})$ rates achieved by our method for the hybrid scenario for arbitrary classes $\F$ can be improved.

There are several new techniques that make it possible to develop  computationally feasible prediction methods with nontrivial regret guarantees:
\begin{itemize}
	\item First is the idea of \emph{relaxations}, presented in  \cite{rakhlin2012relax} for the full-information setting. An extension to partial information case has been a big roadblock for developing new bandit methods. We present this extension here.
	\item Second is the idea of a random playout, also employed in \cite{RakSri15hierarchies}. We show that by having access to unlabeled contexts, the computational (and statistical) difficulty of integrating with respect to the unknown distribution simply disappears.
	\item We extend the notion of classical Rademacher averages to the case of vector-valued functions. The symmetrization technique in this case is of independent interest. 
	\item In many cases, the offline ERM optimization problem (which we assume away as an ``oracle call'') may be NP hard. Building on the technique of \cite{RakSri15hierarchies}, we employ optimization-based relaxations for integer programs. We prove that the regret bound of the resulting algorithm only worsens by a multiplicative factor that is related to the ratio of average widths of the relaxed and the original sets. 
\end{itemize}

It is worth emphasizing again that the family of prediction methods presented in this work is drived from the partial-information extension of the relaxation framework, and the resulting algorithms are distinct from the ones appearing in the literature. We believe that this approach is systematic and can partially fill the gap in our understanding of the algorithmic possibilities for contextual bandits.

\section{Notation}

We denote $[d]\deq\{1,\ldots,d\}$ and $a_{1:t}\deq\{a_1,\ldots,a_t\}$. Let $\Delta_d$ be the probability simplex over $d$ coordinates. The vector of ones is denoted by $\bf{1}$ and an indicator of event $A$ by $\ind{A}$. For a matrix $M$, we use $M_t$ to refer to its $t$-th column.

\section{Setup}
\label{sec:setup}

Let us recall the online protocol. On each round $t\in[n]$, we observe side information $x_t\in\X$, predict $\pred_t\sim q_t\in\Delta_d$, and observe feedback $c_t(\pred_t)$ for some $c_t\in[0,1]^d$. 

Given $x_{1:n}$, it is convenient to work with a matrix representation of the class $\F$ projected on these data. Each $f\in\F$ yields sequence $(f(x_1),\ldots,f(x_n))$, which we collect as a ${d\times n}$ matrix $M_f$, defined as 
\begin{align}
	\label{def:M}
	M_f(j, t) = \ind{f(x_t)=j}.
\end{align}
Let $\Mhat=\Mhat[x_{1:n}]=\{M_f: f\in\F\}$ denote the collection of matrices. (The hat on $\Mhat$ will remind us of the dependence of this set on $x_{1:n}$, even if not explicitly mentioned). 

We may now define the oracle employed by the prediction method:
\begin{definition}
	Given a class $\F$ of policies $\X\to[d]$, a set of covariates $x_{1:n}$, and a real-valued $d\times n$ matrix $Y$, a \emph{value-of-ERM} oracle returns the value  
	\begin{align}
		\label{eq:erm_objective}
		\inf_{M\in\Mhat[x_{1:n}]} ~\sum_{t=1}^n M_t^\tr Y_t ~.
	\end{align}
	The oracle is called \emph{$\delta$-approximate} if the reported value is within $\delta$ from the minimum.
\end{definition}

We may express the comparator term in \eqref{eq:regret} as an ERM objective \eqref{eq:erm_objective} with $Y=[c_1,\ldots,c_n]$. Closely related to this expression is a new (to the best of our knowledge) definition of Rademacher averages for vector-valued functions: given $x_{1:n}$, define
\begin{align}
	\label{def:rad}
	\Rad(\F; x_{1:n}) &\deq \Rad(\Mhat) \deq \En_{\bepsilon_{1:n}} \sup_{M\in\Mhat} \sum_{t=1}^n M_t^\tr \bepsilon_t
\end{align}
where $\bepsilon_1,\ldots,\bepsilon_n$ are $d$-dimensional vectors with independent Rademacher random variables. We observe that Rademacher complexity is nothing but a (negative of) the ERM objective with the random matrix $[-\bepsilon_1,\ldots,-\bepsilon_n]$. Indeed, as in the classical case, correlation of the vector valued function class $\F$ with noise measures its complexity.

\section{Relaxations for Partial Information}
\label{sec:rel_for_partial}

Let us write the information obtained on round $t$ as a tuple $$I_t(x_t,q_t,\pred_t,c_t)=(x_t,q_t,\pred_t,c_t(\pred_t)),$$ 
keeping in mind that $x_t$ is revealed before $q_t$ is chosen. In full information problems, $I_t$ contains the vector $c_t$, but not so in our bandit case. For partial information problems, it turns out to be crucial to include $q_t$ in the definition of $I_t$, in addition to the value $c_t(\pred_t)$.

A \emph{partial-information relaxation} $\Rel{}{}$ is a function that maps $(I_1,\ldots,I_t)$ to a real value, for any $t\in[n]$. We say that the partial-infromation relaxation $\Rel{}{I_1,\ldots,I_t}$ is \emph{admissible} if for any $t\in[n]$, for all $I_1,\ldots,I_{t-1}$,
\begin{align}
	\label{eq:recursive}
	&\Ex_{x_t}\inf_{q_t} \max_{c_t} \Ex_{\pred_t\sim q_t} \left\{ c_t(\pred_t) + \Rel{}{I_{1:t-1}, I_t(x_t,q_t,\pred_t,c_t)}\right\} \leq \Rel{}{I_{1:t-1}}
\end{align}
and for all $x_{1:n}$,$c_{1:n}$, and $q_{1:n}$,
\begin{align}
	\label{eq:initial}
	\Ex_{\pred_{1:n}\sim q_{1:n}}\Rel{}{I_{1:n}} \geq -\inf_{f\in\F} \sum_{t=1}^n f(x_t)^\tr c_t ~.
\end{align}
In the above expressions, $x_t$ follows the (unknown) distribution $P_x$, $q_t$ ranges over distributions on $[d]$, and $c_t$ over $[0,1]^d$.

Any randomized strategy $(q_t)_{t=1}^n$ that certifies the inequalities \eqref{eq:recursive} and \eqref{eq:initial} is called \emph{an admissible strategy}. 
\begin{lemma}
	\label{lem:regret_from_admissibility}
	Let $\Rel{}{}$ be an admissible relaxation and $(q_t)_{t=1}^n$ an admissible strategy. Then for any $c_{1:n}$,
	$$\En[\Reg]\leq \Rel{}{\emptyset}.$$
\end{lemma}

The above partial-information relaxation setup appears to be ``the right'' analogue of the full-information relaxation framework. While we do not present it here, one may recover the EXP4 algorithm through the above approach, with the correct regret bound.

We will now present an admissible strategy for the contextual bandit problem, assuming we can sample from the distribution $P_x$, or have access to unlabeled data.

\section{The BISTRO Algorithm}

For any $t\in[n]$, define a $d\times n$ matrix $Y^{(t)}$ as
\begin{align*}
Y^{(t)} = [c_1,\ldots,c_{t-1},c_t,2\bepsilon_{t+1},\ldots,2\bepsilon_n]
\end{align*}
with $\bepsilon_s\in\{\pm1\}^d$ a vector of independent Rademacher random variables. At each step $t\in[n]$, the randomized method presented below calculates a distribution $q_t\in\Delta_d$ with each coordinate at least $\gamma$ and defines an unbiased estimate $\tilde{c}_t$ of $c_t$ in a usual manner as
$$\tilde{c}_t(j) = \ind{\pred_t=j}\times c_t(\pred_t)/q_t(j).$$
It is standard to verify that $\En_{\pred_t\sim q_t} \tilde{c}_{t} = c_t$. We then define 
\begin{align}
	\label{eq:matrix_of_estimates}
	\tilde{Y}^{(t)} = [\tilde{c}_1,\ldots,\tilde{c}_{t-1},\tilde{c}_t,2\gamma^{-1}\bepsilon_{t+1},\ldots,2\gamma^{-1}\bepsilon_n],
\end{align}
and recall that $\tilde{Y}^{(t)}_s$ denotes the $s$-th column of this matrix. The next theorem is the main result of the paper.
\begin{theorem}
	\label{thm:main}
The partial-information relaxation
\begin{align}
	\label{eq:our_relaxation}
	\Rel{}{I_{1:t}} = \Ex_{(x,\bepsilon)_{t+1:n}}\sup_{M\in{\Mhat}} \left\{-\sum_{s=1}^n M_s^\tr  \tilde{Y}^{(t)}_s \right\} + (n-t)\gamma
\end{align}
is admissible. An admissible randomized strategy for this relaxation is given by BISTRO (Algorithm~\ref{alg}). The expected regret of the algorithm with $\gamma=\sqrt{2\En\Rad(\F;x_{1:n})/(nd)}$ is upper bounded by 
$$2\sqrt{2d \cdot n \cdot \En\Rad(\F;x_{1:n})}.$$
\end{theorem}

\begin{algorithm}[hbtp]
    \caption{BISTRO: \textbf{B}and\textbf{I}t\textbf{S} wi\textbf{T}h \textbf{R}elaxati\textbf{O}ns}\label{alg}
    \begin{algorithmic}[1]
    \INPUT Parameter $\gamma\in (0,1/d)$
    \FOR{ $t=1,\ldots,n$}
    \STATE Observe $x_t$. Draw $x_{t+1:n}\sim P_x$ and $\bepsilon_{t+1:n}$~.
	\STATE Construct $\tilde{Y}^{(t)}$ and define $q^*_t$ to be a minimizer of
	\begin{align*}
		\max_{j\in[d]}\left\{ q^\tr \e_j - \min_{M\in\Mhat[x_{1:n}]}\left\{\sum_{s\neq t} \gamma M_{s}^\tr \tilde{Y}^{(t)}_{s} + M_{t}^\tr \e_j \right\}  \right\}
	\end{align*}
	over $q\in\Delta_d$ and set 
	\begin{align}
		\label{eq:mix_in_unif}
		q_t = (1-\gamma d) q^*_t + \gamma \bf{1}.
	\end{align}
	\STATE Predict $\pred_t\sim q_t$ and observe $c_t(\pred_t)$.
    \STATE Create an estimate $\tilde{c}_t$: $$\tilde{c}_t(j) = \ind{\pred_t=j}\times c_t(\pred_t)/q_t(j).$$
    \ENDFOR
    \end{algorithmic}
\end{algorithm}

The draw $x_{t+1:n}\sim P_x$ can be realized by drawing from a pool of unlabeled data. 

The random signs comprising the matrix $\tilde{Y}$ provide a form of ``regularization''. We remark that in experiments, one may obtain better performance by replacing the factor $2$ in \eqref{eq:matrix_of_estimates} with a smaller value, or even with zero. A theoretical justification for this (which is related to using a surrogate loss) is beyond the scope of this paper.

\begin{lemma}
	\label{lem:waterfilling}
	The calculation of $q_t^*$ in BISTRO\footnote{`Bistro' means `fast' in Russian.} can be done by a water-filling argument and requires $d$ calls to the ERM oracle.
\end{lemma}
\begin{proof}[\textbf{Proof of Lemma~\ref{lem:waterfilling}}]
	The optimization problem in Algorithm~\ref{alg} is of the form
	$$\min_{q\in\Delta_d}\max_{j\in[d]} \{ q_j - \psi_j\}$$
	where $\psi_j$ is the value of the infimum over $\Mhat$ corresponding to $\e_j$, and it is solved by a water-filling argument which we describe next. Each value $\psi_j$ is a value-of-ERM oracle call. Let $\psi_{(1)}\geq \ldots\geq\psi_{(d)}$ be a sorted order of these values, and let $q_{(1)}=\ldots=q_{(d)}=0$ be the initial values of the corresponding coordinates of the solution $q$. Start with a unit amount and assign $q_{(1)}=\psi{(1)}-\psi{(2)}$. Then add $\psi_{(2)}-\psi_{(3)}$ to both $q_{(1)}$ and $q_{(2)}$, and proceed until either the unit mass is exhausted, or the smallest coordinate $(d)$ in the ordering is reached and filled. In the former case, $q$ is the solution, and the latter case requires us to uniformly fill all the coordinates of $q$ until they sum to one. It is easy to see that this procedure minimizes the maximum difference.
\end{proof}

The algorithm only requires the \emph{value} of the ERM objective, not the solution. Furthermore, this value can be $\delta$-approximate, and the additional error is $O(n\delta)$ over the $n$ rounds. This provides extra flexibility, since approximate ERM values may be obtained via optimization methods.

Perhaps the most unusual aspect of the algorithm is the use of unlabeled data. It is an example of a general  random playout idea. In the setting of online linear optimization, the Follow-the-Perturbed-Leader method is an example of such a random playout, yet the idea extends well beyond this scenario. As shown in \cite{rakhlin2012relax}, the random playout technique can be applied when a certain worst-case-choice can be replaced with a known bad-enough distribution. However, when side information $x_t$ is i.i.d., the step is not even required. Furthermore, an inspection of the proof shows that we may deal with $x$'s coming from a non-i.i.d. stochastic process, as long as we are able to draw future samples from it.

We also remark that \eqref{eq:mix_in_unif} may be applied only to the coordinates that are close to zero, if any. The potential suboptimality of the $O(n^{3/4})$ bound stems from the uniform exploration. It is an open question whether this can be improved systematically for all classes $\F$, or whether there is a different structural property that allows one to avoid this form of exploration.

\section{Extensions}
\label{sec:extensions}
In this section, we outline several extensions of BISTRO. Specifically, we show how to incorporate additional data-based constraints, and how to use further optimization-based relaxations (such as LP or SDP), to obtain polynomial time methods for the ERM (or regularized ERM) solution. We show that one obtains a regret bound that only worsens by a factor related to the integrality gap of the integer program relaxation. With an eye on both computation and prediction performance, these techniques expand the applicability of BISTRO.

\subsection{Data-dependent policy classes}

An inspection of the proof reveals that all the steps go through if define regret in \eqref{eq:regret} with respect to a data-dependent class $\F[x_{1:n}]$:
\begin{align}
	\label{eq:regret_data_dep}
	\sum_{t=1}^n q_t^\tr c_t-\inf_{f\in\F[x_{1:n}]} \sum_{t=1}^n f(x_t)^\tr c_t.
\end{align}
In this case, given $x_{1:n}$, to each $f\in\F[x_{1:n}]$ we associate $M_f$ as defined in \eqref{def:M}, and take 
$$\Mhat=\{M_f: f\in\F[x_{1:n}]\}.$$
The BISTRO algorithm is then identical, while the regret upper bound of Theorem~\ref{thm:main} now replaces $\En\Rad(\F; x_{1:n})$ with $\En\Rad(\F[x_{1:n}]; x_{1:n})$.

The ability to change the set of policies according to the actual data allows an extra degree of flexibility. This flexibility can be realized via additional global constraints in terms of $x_{1:n}$, as we show in the next few sections. We also discuss a concrete example.

\subsection{Data-based constraints}

	A particular way to define a data-dependent subset of $\F$ is via constraints. Suppose we let $C(f; x_{1:n})$ be the degree to which $f\in\F$ violates constraints with respect to the given data $x_{1:n}$. We then define
\begin{align}
	\label{eq:F_data_dep}
	\F_K[x_{1:n}] = \{f\in\F: C(f; x_{1:n})\leq K\},
\end{align}
a pruning of the original class that keeps only those policies that do not violate the constraints by more than $K$. Let us give an example.

\paragraph{Example: Product Recommendation}
	Suppose at each time step we are asked to recommend one of $d$ products to a person, based on her covariate information $x_t$. Let $\F$ be a set of policies that map $x_t$ to the particular choice of the product (e.g. the label achieving maximum projection of $x_t$ onto $d$ vectors $w_j$; here $\F$ may consist of all such unit vector tuples). The payoff is whether the person decided to buy the recommended product. However, suppose $x_t$ also encodes the location (physical, or within a network), and we believe it is a good idea to focus recommendations such that near-by people are targeted with the same product. The marketing motivation here is two-fold: first, the recommendations would reinforce each other when individuals communicate, or if one of them buys the product; second, in a social network near-by individuals (friends) tend to have similar tastes, and thus a good policy would suggest similar items. 
	
	The objective of enforcing similarity of recommendations is a global constraint that can only be checked once we know all the $x_1,\ldots,x_n$. We can easily incorporate the constraint into the definition of $\F_K[x_{1:n}]$ as follows. Let $w(x_s,x_r)$ be the cost of providing different recommendations to $x_s$ and $x_r$ (which is smaller if the two individuals are ``far''). In the case of a network, we may set, for instance, $w(x_s,x_r)=0$ if the $s$th person is more than a hop away from the $r$th person. 
Define
	\begin{align}
		\label{eq:cost_of_policy}
		C(f; x_{1:n}) = \sum_{s,r\in[n]} w(x_s,x_r)\ind{f(x_s)\neq f(x_r)},
	\end{align}
the constraint violation by $f$ in assigning products to the given set of individuals. Let $\F_K[x_{1:n}]$ be defined as in \eqref{eq:F_data_dep}. Note that the constraint is not on the behavior of the recommendation engine, but on the set of policies that we hope will do well for the problem. If there is indeed the effect of reinforcement of recommendations or similarity of tastes within the local neighborhood, the restriction to a smaller set $\F_K[x_{1:n}]$ is justified.

Within the same setting of product recommendation, we might instead take a set of policies ensuring that within each neighborhood at least $k$ individuals receive each particular product recommendation. This constraint, which roughly corresponds to ``coverage'' of the relevant population, can be written as
$$C(f; x_{1:n}) = \sum_{\ell} \sum_{j\in[d]} \left[k-\sum_{s\in T_\ell} f(x_s)[j]\right]_+$$
where $\{T_\ell\}_\ell$ is a partition of $[n]$ into neighborhoods according to information contained in $x_{1:n}$. The above two examples give a flavor of the constraints that can be encoded --- the framework is flexible enough to fit a wealth of scenarios. 

From the computational point of view, it might be difficult to obtain the ERM value over a constrained set $\F_K[x_{1:n}]$. Instead, we consider an additional form of relaxation, where the constraint is subtracted off as a Lagrangian term. We will then employ certain linear programming relaxations to solve the product recommendation problem. Notably, by going to a regularized version of relaxations we are not changing the regret definition, which is still with respect to the constrained set.

\subsection{Regularized relaxation}

Let $\F_K[x_{1:n}] = \{f\in\F: C(f; x_{1:n})\leq K\}$ be the constrained set for some value $K$ and a constraint function $C$, as in the previous section. Let us write $C(M; x_{1:n})$ for the matrix representation the corresponding $f\in\F$. The following form of a relaxation may be better suited for approximation algorithms than the one where the constraint is strictly enforced.
\begin{lemma} 
	\label{lem:reg_version}
	For any $\lambda, K>0$, the partial-information relaxation 
\begin{align}
	\label{eq:regularized_relaxation}
	&\Ex_{(x,\bepsilon)_{t+1:n}}\sup_{M\in{\Mhat}} \left\{-\sum_{s=1}^n M_s^\tr  \tilde{Y}^{(t)}_s - \lambda C(M; x_{1:n})\right\} \notag\\
	&~~~~~~~~~~~~~ + \lambda K +  (n-t)\gamma
\end{align}
is admissible, where $\Mhat$ denotes the matrix representation of the \emph{original (unconstrained)} set $\F$ of policies.
\end{lemma}
\begin{proof}[\textbf{Proof of Lemma~\ref{lem:reg_version}}]
	We check that the initial condition is satisfied. For this purpose, let $\Mhat_K$ be the set of matrices corresponding to the constrained set $\F_K[x_{1:n}]$. Similarly to \eqref{eq:proof_initial_cond} in the proof of Theorem~\ref{thm:main},
	\begin{align*}
		&-\inf_{f\in\F_K[x_{1:n}]} \sum_{t=1}^n f(x_t)^\tr c_t \leq \Ex \sup_{M\in\Mhat_K}\sum_{t=1}^n -M_t^\tr \tilde{Y}^{(n)}_t \leq \Ex \sup_{M\in\Mhat}\left\{ \sum_{t=1}^n -M_t^\tr \tilde{Y}^{(n)}_t - \lambda C(M; x_{1:n}) \right\} +\lambda K. 
	\end{align*}
	The second inequality holds since all the matrices in the former supremum have the constraint value bounded by $K$. The recursive condition argument follows exactly as in the proof of Theorem~\ref{thm:main}.
\end{proof}
The only change required for BISTRO is to define the optimization objective in terms of \emph{regularized} ERM values
\begin{align}
	\label{eq:penalized_ERM}
		\min_{M\in\Mhat}\left\{\sum_{s\neq t} \gamma M_{s}^\tr \tilde{Y}^{(t)}_{s} + M_{t}^\tr \e_j + \gamma^{-1}\lambda C(M; x_{1:n})\right\}
\end{align}
over the \emph{unconstrained} set of matrices corresponding to $\F$. While the required minimization problem is over an unconstrained set of policies, we can control the expected regret 
\begin{align}
	\sum_{t=1}^n q_t^\tr c_t-\inf_{f\in\F_K[x_{1:n}]} \sum_{t=1}^n f(x_t)^\tr c_t.
\end{align}
of the modified BISTRO with respect to the \emph{constrained} set $\F_K[x_{1:n}]$, which is the original goal. The regret is given by $\Rel{}{\emptyset}$, which is at most
$$\En \sup_{M\in \Mhat} \left\{ - \gamma^{-1}\sum_{t=1}^n M_t^\tr \bepsilon_t - \lambda C(M;x_{1:n})\right\} + nd\gamma + \lambda K.$$
It is possible to optimally balance $\lambda$ with respect to $K$ and the Rademacher averages in a data-driven manner, but we omit this step for brevity.

As we illustrate in the next section, optimization problems of the form \eqref{eq:penalized_ERM} may admit a linear programming (or other) relaxation, offering an alternative to the optimization problem over the constrained set.

\subsection{Optimization-based relaxations}

To make the algorithm of this paper more applicable, we discuss here the situation where the ERM oracle or the regularized ERM oracle for the class $\F_K[x_{1:n}]$ (or the unconstrained set $\F$) is a  difficult or even an NP-hard integer program. The idea is to choose a superset $\widetilde{\cM}\supseteq\Mhat$ for which the linear optimization problem is easier. 
\begin{lemma}
	\label{lem:opt_relaxation}
	Let $\widetilde{\cM}\supseteq\Mhat$ be a set of matrices such that the column sum $\sum_{j=1}^d M_t(j) \leq 1$ for any $M\in \widetilde{\cM}$ and $t\in[n]$. Then the partial information relaxation
	\begin{align*}
		\Rel{}{I_{1:t}} = \Ex_{(x,\bepsilon)_{t+1:n}}\sup_{M\in\widetilde{\cM}} \left\{-\sum_{s=1}^n M_s^\tr  \tilde{Y}^{(t)}_s \right\} + (n-t)\gamma
	\end{align*}
	is admissible. BISTRO (with ERM over $\widetilde{\cM}$ rather than $\cM$) is an admissible strategy for this relaxation and the expected regret is upper bounded by 
	$$2\sqrt{2d\cdot n\cdot \En\Rad(\widetilde{\cM})}.$$
	Similarly, using $\widetilde{\cM}$ in \eqref{eq:regularized_relaxation} yields an admissible relaxation, and BISTRO with the corresponding regularized ERM is an admissible strategy.
\end{lemma}
The set $\widetilde{\cM}[x_{1:n}]$ may be defined via linear programming or SDP relaxations for integer programs, or via Lasserre/Parrilo hierarchies \cite{lasserre2001global,parrilo2003semidefinite}. There is a large body of literature that aims at understanding the integrality gap in relaxing the integer program. These results are directly applicable to the present problem. 

As a concrete example, consider the product recommendation example in the previous section, and consider the cost \eqref{eq:cost_of_policy} for each policy and the restriction $\F_K[x_{1:n}]$ in \eqref{eq:F_data_dep}. We assume here that $\F$ is the set of all possible labelings, since in general the optimization problem will depend on the structure of $\F$ and its description. Let us phrase the regularized ERM integer program \eqref{eq:penalized_ERM} as a \emph{Metric Labeling Constraint} \cite{kleinberg2002approximation} problem. The general form of this integer program is given for $z\in[d]^n$ by
\begin{align}
	\label{eq:MLC}
	g(z) = \sum_{v\in V} d_1(v, z_v) + \sum_{(u,v)\in E} W_{(u,v)} d_2(z_u,z_v)
\end{align}
where $G=(V,E,W)$ is a graph with nonnegative weights, $|V|=n$, the value  $d_1:V\times [d]\to \reals$ is a cost of assigning a label to a node, and the separation cost $d_2:[d]\times[d]\to \reals_{\geq 0}$ on the edges is a metric on the space of labels. The Metric Labeling Constraint problem asks for a solution that minimizes $g(z)$ over $[d]^n$.

For our application to product recommendation we convert the regularized minimization objective of \eqref{eq:penalized_ERM} with the constraint \eqref{eq:cost_of_policy} into the above form \eqref{eq:MLC} by matching the assignment costs to the linear part and the separation costs to the constraint part \eqref{eq:cost_of_policy}. More precisely, let $G$ be a fully connected graph with weights $W_{(s,r)} = \gamma^{-1}\lambda \cdot w(x_s,x_r)$ between nodes corresponding to $x_s$ and $x_r$. The indices of vertices correspond to time steps in $[n]$, and $z_v$ corresponds to the coordinate chosen by the particular $M$ at time $v$. We take $d_1(v,z_v)$ to be the value $\gamma\e_{z_v}^\tr \tilde{Y}^{(t)}_v$ if $v\neq t$ and  $\e_{z_v}^\tr \e_j$ if $v=t$. Define $d_2(a,b) = \ind{a\neq b}$ to be the uniform metric. We may also define a metric on the space of products, assigning smaller distance to similar items.

\citep{kleinberg2002approximation} give an LP relaxation for the Metric Labeling Constraint problem. The set that defines the relaxation is precisely the set $\widetilde{\cM}$ we seek. Furthermore, the authors prove a $2$-approximation ratio for the uniform metric, which is the case here. (\cite{chekuri2004linear} prove an integrality gap of $O(\log k)$ for the general case). 

Given the $2$-approximation ratio result, we conclude that the regret bound for BISTRO with the LP program as the relaxation of the regularized ERM is only a constant worse than the bound with the constrained set $\F_K[x_{1:n}]$. The exact optimization over the latter set may be computationally intractable, while we provide an efficient method to achieve a bound, optimal to within a constant. As already noted in \cite{RakSri15hierarchies}, such an approach that fuses approximation algorithms and online relaxations is able to produce polynomial-time methods with regret defined as $1\times$ the benchmark, while the benchmark itself may be NP-hard. This phenomenon can be attributed to the improper nature of the predictions, which need not be consistent with any particular policy in $\F$. 

More generally, by obtaining a multiplicative approximation of $\gap$ for the integer program, one may derive 
\begin{align}
	\label{eq:relax_rademacher}
	\En\Rad(\widetilde{\cM}[x_{1:n}])\leq O(\gap) \times \En\Rad(\cM[x_{1:n}]).
\end{align}
Then one obtains a method with better computational properties and a regret bound which is only $O(\sqrt{\gap})$ worse. Once again, the factor in front of the comparator in the definition \eqref{eq:regret} of regret is still one when using $\widetilde{\cM}$ as a relaxation.  

Finally, we remark that \eqref{eq:relax_rademacher} is comparing an \emph{average} width of $\widetilde{\cM}$ (largest projection onto noise) with an average width of $\cM$. Such a comparison of average widths  (and, therefore, ``average gap'') for useful sets of contextual bandit policies $\F$ appears to be an interesting area of further investigation. We refer to \cite{RakSri15hierarchies}, where some of these ideas have been developed in the context of cut-based constraints for node prediction on graphs.

\subsection{Adversarial contexts}

Suppose we place no assumption on the evolution of $x_t$'s, which may now be treated as worst-case. This problem subsumes the full information online classification setting, and, hence, one cannot hope to have nontrivial regret against policy classes $\F$ with infinite Littlestone dimension. More generally, the best one can hope for is to say that the adversarial contextual bandit problem can be solved whenever the corresponding full information problem may be solved. We now present essentially this result: if there is a full-information relaxation, then one may use it to solve the adversarial contextual bandit problem. Moreover, based on the work of \cite{rakhlin2012relax,FosRakSri15}, all the known online learning methods appear to be relaxation based. Hence, we essentially prove below that
\begin{quote}
	If a problem is online learnable in the  full-information adversarial setting, then it is learnable in the adversarial contextual bandit setting. Furthermore, if the former is computationally tractable, then so is the latter.
\end{quote}

To be precise, the full information version of contextual problem is as follows. On round $t$, we observe $x_t\in\X$, predict $\pred_t\in [d]$, and observe $c_t\in[0,1]^d$. The regret is defined as before, with our cumulative cost being $\sum c_t(\pred_t)$.

A full information relaxation $\Relfull{}{c_1,\ldots,c_t}$ is admissible if
\begin{align*}
	&\sup_{x_t}\inf_{q_t} \max_{c_t} \Ex_{\pred_t\sim q_t} \left\{ c_t(\pred_t) + \Relfull{}{c_{1:t}}\right\} \leq \Relfull{}{c_{1:t-1}}
\end{align*}
and
\begin{align*}
	\Relfull{}{c_{1:n}} \geq -\inf_{f\in\F}\sum_{t=1}^n f(x_t)^\tr c_t ~.
\end{align*}
Similarly, a partial information relaxation is admissible in this adversarial case when $c_{1:t}$ are replaced with $I_{1:t}$ in the above admissibility definition, as in Section~\ref{sec:rel_for_partial}.

\begin{lemma}
	\label{lem:if_full_then_partial}
	If $\Relfull{}{}$ is an admissible full-information relaxation for the adversarial scenario, then 
	$$\Rel{}{I_{1:t}} \deq \gamma^{-1}\Relfull{}{\gamma\tilde{c}_{1},\ldots,\gamma\tilde{c}_t}+(n-t)d\gamma$$ 
	is admissible for the partial information scenario. Prediction $q_t$ is obtained as $q_t=(1-d\gamma)q_t^* + \gamma \bf{1}$ where $q_t^*$ is computed by solving for a full-information strategy with the scaled unbiased estimates of costs. The resulting regret upper bound is 
	$$2\sqrt{d \cdot n \cdot \Relfull{}{\emptyset}}.$$
\end{lemma}
\begin{proof}[\textbf{Proof of Lemma~\ref{lem:if_full_then_partial}}]
	 Let us first check the initial condition. We have that
	 \begin{align*}
		 &\Ex_{\pred_{1:n}\sim q_{1:n}} \Rel{}{I_{1:n}} = \Ex_{\pred_{1:n}\sim q_{1:n}} \gamma^{-1}\Relfull{}{\gamma\tilde{c}_{1},\ldots,\gamma\tilde{c}_{n}} \\
		 &\geq \Ex_{\pred_{1:n}\sim q_{1:n}} -\inf_{f\in\F} \sum_{t=1}^n f(x_t)^\tr \tilde{c}_t 
		 \geq -\inf_{f\in\F} \sum_{t=1}^n f(x_t)^\tr c_t
	\end{align*}
	where the first inequality is due to admissibility of the full-information relaxation, and the second is due to Jensen's inequality and unbiasedness of $\tilde{c}_t$. For the recursive part, we follow the proof of Theorem~\ref{thm:main} and note that all the statements, until the end, are done conditionally on $x_t$. Define the strategy $q_t^*$ as 
	$$q_t^* = \argmin{q\in\Delta_d} \sup_{\tilde{c}\in\gamma^{-1}[0,1]^d} \left\{ q^\tr (\gamma \tilde{c}_t) + \Relfull{}{\gamma\tilde{c}_{1},\ldots,\gamma\tilde{c}_{t}} \right\}$$
	and let $q_t=(1-d\gamma)q_t^* + \gamma \bf{1}$. Given $x_t$, \eqref{eq:rel_t_1} tells us
	\begin{align*}
		&\max_{c_t\in[0,1]^d} \En_{\pred_t\sim q_t} \big\{ c_t(\pred_t) + \Rel{}{I_1,\ldots,I_t}\big\} \le \sup_{\tilde{c}_t\in\gamma^{-1}[0,1]^d} \left\{  (q_t^*)^\tr  \tilde{c}_t + \Rel{}{I_1,\ldots,I_t}\right\} +d\gamma 
	\end{align*}
	which is equal to
	\begin{align*}
		&\gamma^{-1}\sup_{\tilde{c}_t} \left\{  (q_t^*)^\tr  (\gamma\tilde{c}_t) + \Relfull{}{\gamma\tilde{c}_{1}, \ldots,\gamma\tilde{c}_{t}} \right\} +(n-t+1)d\gamma\\
		&\leq \gamma^{-1}\Relfull{}{\gamma\tilde{c}_{1}, \ldots,\gamma\tilde{c}_{t-1}} + (n-t+1)d\gamma
	\end{align*}
	by admissibility of the full-information relaxation. Observe that the use of the full-information relaxation on $\gamma \tilde{c}_t$'s is warranted since these vectors are in $[0,1]^d$. This concludes the proof.
\end{proof}
We remark that the time complexity of the adversarial contextual bandit solution in Lemma~\ref{lem:if_full_then_partial} is the same as the time complexity of the corresponding full information procedure.

\section{Open Problems and Future Directions}

The main open problem is whether the regret upper bound for BISTRO or a related method can be improved. In the inequality \eqref{eq:rel_t_1} we decouple the distribution $q_t'$ from $q_t$, and this appears to be the source of the loseness, at least in the analysis. A more precise analysis at this step might resolve the issue. It is unclear what kind of structure of $\F$ may be used to improve computation and/or regret guarantees of BISTRO. 

Under structural assumptions on $\F$ one may come up with sufficient statistics for the information $I_{1:t}$ and, therefore, avoid keeping around all the estimates $\tilde{c}_t$. Of course, this is the case in non-contextual bandits, where the sum $\sum\tilde{c}_t$ is sufficient (at least as evidenced by existing near-optimal bandit methods).

An interesting avenue of investigation is to study the more general case when $x$'s are drawn from a stochastic process with a parametrized form. One may then attempt to estimate the parameters of the process on-the-go and use the estimate to hallucinate future data for random playout.

\section{Proofs}

\begin{proof}[\textbf{Proof of Lemma~\ref{lem:regret_from_admissibility}}]
	In the proof, we use the shorthand $\multiminimax{\ldots}_{t=1}^n$ do denote repeated application of the operators within the brackets from $t=1$ to $n$. As an example, the sequence of operators $$\Ex_{x_1}\max_{c_1}\Ex_{x_2}\max_{c_2} [G(x_1,c_1,x_2,c_2)]$$ acting on the function $G$ is abbreviated as $\multiminimax{\Ex_{x_t}\max_{c_t}}_{t=1}^2[G(x_1,c_1,x_2,c_2)].$ 
	
	Let $q_1,\ldots,q_n$ be an admissible strategy. The expected regret of this strategy can be upper bounded by
	\begin{align*}
		\En[\Reg] \leq \sup_{c_{1:n}}\En[\Reg] \leq \multiminimax{\Ex_{x_t}\sup_{c_t}}_{t=1}^n \left[ \sum_{t=1}^n q_t^\tr c_t-\inf_{f\in\F} \sum_{t=1}^n f(x_t)^\tr c_t\right] 
	\end{align*}
	by Jensen's inequality (pulling $\Ex_{x_t}$ out of multiple suprema until its $t$-th position). The last expression is further upper bounded by
	\begin{align*}
		\multiminimax{\Ex_{x_t}\sup_{c_t}}_{t=1}^n \left[ \sum_{t=1}^n q_t^\tr c_t+ \Ex_{\pred_{1:n}\sim q_{1:n}}\Rel{}{I_{1:n}}\right] 
	\end{align*}
	by admissibility of the partial information relaxation. By linearity of expectation for $\Ex_{\pred_t}$ and Jensen's inequality (to pull it out through multiple suprema as before), we obtain an upper bound of
	\begin{align*}
				 &\multiminimax{\Ex_{x_t}\sup_{c_t}\Ex_{\pred_{t}\sim q_{t}}}_{t=1}^n \left[ \sum_{t=1}^n c_t(\pred_t)+ \Rel{}{I_{1:n}}\right]. 
	\end{align*}
	We now start from step $n$ and observe that $\sum_{t=1}^{n-1} c_t(\pred_t)$ does not depend on $x_n,c_n,\pred_n$, and thus we rewrite the preceding expression as 
	\begin{align*}
		 &\multiminimax{\Ex_{x_t}\sup_{c_t}\Ex_{\pred_{t}\sim q_{t}}}_{t=1}^{n-1} \left[ \sum_{t=1}^{n-1} c_t(\pred_t) + \Ex_{x_t}\sup_{c_t}\Ex_{\pred_{t}\sim q_{t}} \left\{c_n(\pred_n) +  \Rel{}{I_{1:n}}\right\}\right] .
	\end{align*}
	By admissibility of $q_t$  and \eqref{eq:recursive}, we pass to the upper bound of 
	\begin{align*}
		 &\multiminimax{\Ex_{x_t}\sup_{c_t}\Ex_{\pred_{t}\sim q_{t}}}_{t=1}^{n-1} \left[ \sum_{t=1}^{n-1} c_t(\pred_t) +  \Rel{}{I_{1:n-1}}\right].
	\end{align*}
	Continuing in this fashion leads to a bound of $\Rel{}{\emptyset}$.
\end{proof}

\bibliography{paper}
\bibliographystyle{alpha}

\appendix

\section{Proof of Theorem~\ref{thm:main}}

\paragraph{Admissibility: initial condition} For any $c_{1:n},q_{1:n},x_{1:n}$, it holds that
\begin{align}
	\label{eq:proof_initial_cond}
	-\inf_{f\in\F} \sum_{t=1}^n f(x_t)^\tr c_t &= \sup_{M\in{\cM}[x_{1:n}]} -\sum_{t=1}^n M_t^\tr Y^{(n)}_t \leq \En_{\pred_{1:n}\sim q_{1:n}} \sup_{M\in{\cM}[x_{1:n}]} -\sum_{s=1}^n M_s^\tr \tilde{Y}^{(n)}_s = \En_{\pred_{1:n}\sim q_{1:n}} \Rel{}{I_{1:n}}.
\end{align}
In the remainder of the proof we will often write $\cM$ instead of $\cM[x_{1:n}]$ for brevity.

\paragraph{Admissibility: recursion} Let $\cD\deq\{\gamma^{-1}\e_j: j\in[d]\}\cup \{\bf{0}\}$, the set of scaled standard basis vectors, together with the origin. Observe that $\tilde{c}_t\in \conv(\cD)$ by our definition of unbiased estimates (in fact, it is only a scaling of one coordinate).

We now reason conditionally on $x_t$. As before, let $\bepsilon_s\in\{\pm1\}^d$ denote a vector of independent Rademacher random variables. Let us abbreviate by $\brho=(\bepsilon_{t+1:n},x_{t+1:n})$, a draw of independent Rademacher variables and covariates from $P_x$ for the ``future rounds'', as part of the random playout procedure. Together with the estimates $\tilde{c}_{s}$ for $s<t$, we may now construct $\tilde{Y}^{(t)}$ and $M$ matrices and define the randomized prediction algorithm as
\begin{align}
q_t^*(\brho) &= \argmin{q \in \Delta_d}\sup_{\tilde{c}\in \cD}\left\{ q^\tr \tilde{c} + \sup_{M\in{\cM}[x_{1:n}]} -\sum_{s\neq t} M_{s}^\tr \tilde{Y}^{(t)}_{s} - M_{t}^\tr \tilde{c}\right\} \label{eq:def_q_opt}\\
	&=\argmin{q \in \Delta_d}\sup_{\pred_t , q'_t}\max_{c_t}\left\{ q^\tr \tilde{c}_t(c_t,q'_t,\pred_t) + \sup_{M\in{\cM}[x_{1:n}]} -\sum_{s\neq t} M_{s}^\tr \tilde{Y}^{(t)}_{s} - M_{t}^\tr \tilde{c}_t(c_t,q'_t,\pred_t)\right\}
\end{align}
We remark that $x_t$ enters the above definition of $q_t^*(\brho)$, but we leave this dependence implicit until the end of the proof. For the purposes of the proof also define
\begin{align}
	\label{eq:step_away}
	q_t(\brho) = (1-d\gamma) \cdot q_t^*(\brho) + \gamma\mathbf{1},
\end{align}
a version of $q_t^*(\brho)$ that is shifted away from the boundary of the simplex (a step that allows for estimation of $c_t$). Also define $q_t = \En_\brho[q_t(\brho)]$ and $q^* = \En_{\brho} [q_t^*(\brho)]$. Observe that 
$$\En_{\pred_t\sim q_t} [c_t(\pred_t)] = q_t^\tr c_t \leq (q_t^*)^\tr c_t + \gamma\mathbf{1}^\tr c_t \leq \En_{\pred_t\sim q_t} [(q_t^*)^\tr \tilde{c}_t(c_t,q_t,\pred_t)] + d\gamma$$ 
Hence,
\begin{align}
	&\max_{c_t\in[0,1]^d} \En_{\pred_t\sim q_t} \big\{ c_t(\pred_t) + \Rel{}{I_1,\ldots,I_t}\big\} \notag\\
	&\leq \max_{c_t\in[0,1]^d} \En_{\pred_t\sim q_t} \big\{  (q_t^*)^\tr \tilde{c}_t(c_t,q_t,\pred_t) + \Rel{}{I_{1:t-1}, I_t(x_t,q_t,\pred_t,c_t)}\big\}  + d\gamma \notag\\
	&\le \sup_{\pred_t \in[d], q'_t}~ \max_{c_t\in[0,1]^d} \left\{  (q_t^*)^\tr  \tilde{c}_t(c_t,q'_t,\pred_t) + \Rel{}{I_{1:t-1}, I_t(x_t,q'_t,\pred_t,c_t)}\right\} +d\gamma. \label{eq:rel_t_1}
\end{align}
In the last expression, the supremum is over $q'_t$ of the form $ (1-d\gamma) \cdot q + \gamma\mathbf{1}$, $q\in\Delta_d$. This last upper bound holds because $q_t$ is one of such distributions. The importance of this upper bound is that it decouples the $q^*_t$ from $q'_t$ in the first term, a step that yields a simple optimization problem that defines $q_t^*(\brho)$. Writing out the form of the relaxation, the last expression is equal to
\begin{align}
	\sup_{\pred_t , q'_t}\max_{c_t} &\left\{ (q_t^*)^\tr \tilde{c}_t(c_t,q'_t,\pred_t) + \En_{\brho}\sup_{M\in{\cM}} -\sum_{s\neq t} M_{s}^\tr \tilde{Y}^{(t)}_{s} - M_{t}^\tr \tilde{c}_t(c_t,q'_t,\pred_t)\right\} + (n-t+1)d\gamma \notag \\
	&\leq \sup_{\tilde{c}_t \in \conv(\cD)}\left\{ (q_t^*)^\tr \tilde{c}_t + \En_{\brho}\sup_{M\in{\cM}} -\sum_{s\neq t} M_{s}^\tr \tilde{Y}^{(t)}_{s} - M_{t}^\tr \tilde{c}_t\right\} + (n-t+1)d\gamma \notag 
\end{align}
since $\tilde{c}_t(c_t,q_t',\pred_t)\in\conv(\cD)$. The expression inside the supremum is a convex function of $\tilde{c}_t$, and thus the supremum is achieved at a vertex, an element of $\cD$. Since $q_t^*=\En_\brho[q_t^*(\brho)]$, we upper bound the last expression via Jensen's inequality (omitting $(n-t+1)d\gamma$ to simplify the exposition) by
\begin{align}
	&\En_{\brho} \sup_{\tilde{c}_t\in \cD}\left\{ q_t^*(\brho)^\tr \tilde{c}_t + \sup_{M\in{\cM}} -\sum_{s\neq t} M_{s}^\tr \tilde{Y}^{(t)}_{s} - M_{t}^\tr \tilde{c}_t\right\} \label{eq:pl1}
\end{align}
Since $q_t^*(\brho)$ is precisely defined to be the minimizer (given $\brho$) of the supremum in \eqref{eq:pl1}, the preceding expression is equal to
\begin{align*}
&\En_{\brho} \inf_{q\in\Delta_d} \sup_{\tilde{c}_t\in \cD}\left\{ q^\tr \tilde{c}_t + \sup_{M\in{\cM}} -\sum_{s\neq t} M_{s}^\tr \tilde{Y}^{(t)}_{s} - M_{t}^\tr \tilde{c}_t\right\} 
\end{align*}
The rest of the upper bounds will be derived conditionally on $\brho$. Observe that 
\begin{align*}
		&\inf_{q\in\Delta_d} \sup_{\tilde{c}_t\in \cD}\left\{ q^\tr \tilde{c}_t + \sup_{M\in{\cM}} -\sum_{s\neq t} M_{s}^\tr \tilde{Y}^{(t)}_{s} - M_{t}^\tr \tilde{c}_t\right\} =\sup_{p_t}  \inf_{q} \En_{\tilde{c}_t \sim p_t}\left\{ q^\tr \tilde{c}_t + \sup_{M\in{\cM}} -\sum_{s\neq t} M_{s}^\tr \tilde{Y}^{(t)}_{s} - M_{t}^\tr \tilde{c}_t\right\} 
\end{align*}
by the minimax theorem, where $p_t$ ranges over the set of distributions on $\cD$. By linearity of expectation, the preceding expression is equal to 
\begin{align}
	&\sup_{p_t}  \inf_{q} \left\{ q^\tr \En_{\tilde{c}_t \sim p_t}[\tilde{c}_t] + \En_{\tilde{c}_t \sim p_t}\sup_{M\in{\cM}} -\sum_{s\neq t} M_{s}^\tr \tilde{Y}^{(t)}_{s} - M_{t}^\tr \tilde{c}_t\right\} \notag\\
	&=\sup_{p_t} \left\{ \min_{j\in[d]}\e_j^\tr \En_{\tilde{c}_t \sim p_t}[\tilde{c}_t] + \En_{\tilde{c}_t \sim p_t}\sup_{M\in{\cM}} -\sum_{s\neq t} M_{s}^\tr \tilde{Y}^{(t)}_{s} - M_{t}^\tr \tilde{c}_t\right\} \label{eq:intrm1}.
\end{align}
Observe that for any $M\in\cM$, $\sum_{j=1}^d M_{j,t} = 1$ and the elements of $M_{t}$ are nonnegative. Thus
\begin{align*}
	 \min_j \e_j^\tr \En_{\tilde{c}_t \sim p_t}[\tilde{c}_t] \leq M_{t}^\tr \En_{\tilde{c}_t \sim p_t}[\tilde{c}_t]
\end{align*}
Therefore, \eqref{eq:intrm1} is equal to
\begin{align*}
	&\sup_{p_t} \left\{ \En_{\tilde{c}_t \sim p_t}\sup_{M\in{\cM}} -\sum_{s\neq t} M_{s}^\tr \tilde{Y}^{(t)}_{s} + \min_{j\in[d]}\e_j^\tr \En_{\tilde{c}_t \sim p_t}[\tilde{c}_t] - M_{t}^\tr \tilde{c}_t\right\} \\
&\le \sup_{p_t} \left\{ \En_{\tilde{c}_t \sim p_t}\sup_{M\in{\cM}} -\sum_{s\neq t} M_{s}^\tr \tilde{Y}^{(t)}_{s} + M_{t}^\tr \En_{\tilde{c}_t \sim p_t}[\tilde{c}_t] - M_{t}^\tr \tilde{c}_t\right\} \\
&= \sup_{p_t} \left\{ \En_{\tilde{c}_t, \tilde{c}'_t\sim p_t}\sup_{M\in{\cM}} -\sum_{s\neq t} M_{s}^\tr \tilde{Y}^{(t)}_{s} + M_{t}^\tr (\tilde{c}'_t -  \tilde{c}_t)\right\} 
\end{align*}
Since exchanging $\tilde{c}_t$ and $\tilde{c}'_t$ switches the sign in the last term, we may introduce an independent Rademacher random variable $\delta_t$ via the standard technique of symmetrization. The last expression is then equal to
\begin{align*}
&\sup_{p_t} \left\{ \En_{\tilde{c}_t, \tilde{c}'_t\sim p_t}\En_{\delta_t}\sup_{M\in{\cM}} -\sum_{s\neq t} M_{s}^\tr \tilde{Y}^{(t)}_{s} + \delta_t M_{t}^\tr (\tilde{c}'_t -  \tilde{c}_t)\right\} \\
&\leq \sup_{p_t} \left\{ \En_{\tilde{c}_t\sim p_t}\En_{\delta_t}\sup_{M\in{\cM}} -\sum_{s\neq t} M_{s}^\tr \tilde{Y}^{(t)}_{s} + 2\delta_t M_{t}^\tr \tilde{c}_t \right\} 
\end{align*}
The above inequality follows by splitting the supremum into two parts equal parts. Let us now reason conditionally on $\tilde{c}_t$. There are two cases: either $\tilde{c}_t=\bf{0}$ or $\tilde{c}_t = \gamma^{-1}\e_{j}$ for some coordinate $j\in[d]$. Let us consider the second case, and the first follows from the same reasoning. Take $Z$ to be a random vector with independent coordinates and values in $\{-\gamma^{-1},\gamma^{-1}\}^d$. For the $j$th coordinate, $Z_j$ is identically $\gamma^{-1}$, while for all other coordinates $i\neq j$ the distribution $Z_i$ is symmetric. Clearly, $\En Z = \tilde{c}_t$. By Jensen's inequality,
\begin{align*}
	\En_{\delta_t}\sup_{M\in{\cM}} \left\{ -\sum_{s\neq t} M_{s}^\tr \tilde{Y}^{(t)}_{s} + 2\delta_t M_{t}^\tr \tilde{c}_t  \right\}
	&\leq \En_{\delta_t}\En_{Z}\sup_{M\in{\cM}} \left\{ -\sum_{s\neq t} M_{s}^\tr \tilde{Y}^{(t)}_{s} + 2\delta_t M_{t}^\tr Z  \right\}
\end{align*}
It is not hard to see that the distribution of $\delta_t Z$ is uniform on $\{-\gamma^{-1}, \gamma^{-1}\}^d$, and we can write it as $\gamma^{-1}\bepsilon_{t}$, a scaled vector of independent Rademacher random variables. The overall bound (together with the omitted term $(n-t +1)d\gamma$) is then
\begin{align*}
		\max_{c_t\in[0,1]^d} \En_{\pred_t\sim q_t} \big\{ c_t(\pred_t) + \Rel{}{I_1,\ldots,I_t}\big\} 
		&\leq \En_{\brho} \sup_{p_t} \left\{ \En_{\tilde{c}_t\sim p_t}\En_{\bepsilon_{t}}\sup_{M\in{\cM}} -\sum_{s\neq t} M_{s}^\tr \tilde{Y}^{(t)}_{s} + 2\gamma^{-1} M_{t}^\tr \bepsilon_t \right\} + (n-t +1)d\gamma\\
	&=\En_{\brho} \En_{\bepsilon_{t}}\sup_{M\in{\cM}} \left\{ -\sum_{s\neq t} M_{s}^\tr \tilde{Y}^{(t)}_{s}  + 2\gamma^{-1} M_{t}^\tr \bepsilon_{t} \right\} + (n-t +1)d\gamma
\end{align*}
since the expression no longer depends on $p_t$ and $\tilde{c_t}$. The above inequality holds for any $x_t$. Hence, we may take expectation on both sides, yielding
\begin{align*}
	\En_{x_t}\max_{c_t\in[0,1]^d} \En_{\pred_t\sim q_t} \big\{ c_t(\pred_t) + \Rel{}{I_1,\ldots,I_t}\big\} 
	&\leq \En_{\bepsilon_{t:n},x_{t:n}} \sup_{M\in{\cM}[x_{1:n}]} \left\{ -\sum_{s\neq t} M_{s}^\tr \tilde{Y}^{(t)}_{s}  + 2\gamma^{-1} M_{t}^\tr \bepsilon_{t} \right\} + (n-t +1)d\gamma \\
	&=\Rel{}{I_{1:t-1}}
\end{align*}
because $\brho=(\bepsilon_{t+1:n},x_{t+1:n})$. This proves admissibility.

\paragraph{Omitting $\bf{0}$ from objective} Examining the algorithm in \eqref{eq:def_q_opt}, we note that the optimization problem may be taken over $\tilde{c}\in\{\e_1,\ldots,\e_d\}$; that is, the $\argmin{}$ over $q$ does not change upon the removal of $\bf{0}$. To see this, suppose that $q_t^*(\brho)$ is the optimal response when $\tilde{c}\in\{\e_1,\ldots,\e_d\}$. Then it is also an optimal response to $\tilde{c}\in\{\e_1,\ldots,\e_d\}\cup\{\bf{0}\}$ since for $\tilde{c}=\bf{0}$ the value of $q$ does not make any difference in terms of the value. This proves our claim, and is reflected in the definition of Algorithm~\ref{alg}.

\paragraph{Regret bound}
The final bound is given by 
$$\Rel{}{\emptyset} = \En_{x}\En_\epsilon \sup_{M\in{\cM}[x_{1:n}]}- \sum_{t=1}^n M_t^\tr \tilde{Y}^{(0)}_t + nd\gamma = \frac{2}{\gamma}\En\Rad(\F; x_{1:n}) + nd\gamma = 2\sqrt{2dn \En\Rad(\F; x_{1:n})}$$

\end{document}